\documentclass[10pt,twocolumn,letterpaper]{article}

\usepackage[accsupp]{axessibility}
\usepackage{iccv}
\usepackage{times}
\usepackage{epsfig}
\usepackage{graphicx}
\usepackage{amsmath}
\usepackage{amssymb}

\usepackage{booktabs}
\usepackage{enumitem}
\usepackage{amsthm}
\theoremstyle{definition}
\newtheorem{definition}{Definition}

\newtheorem{proposition}{Proposition}

\def\parasmall#1{\smallskip\noindent\textbf{#1\,}}
\def\para#1{\medskip\noindent\textbf{#1\,}}

\usepackage[pagebackref,breaklinks,colorlinks,citecolor=blue,bookmarks]{hyperref}
\usepackage[sort,nocompress]{cite}

\iccvfinalcopy

\def\iccvPaperID{12294}

\begin{document}

\title{Improving Equivariance in State-of-the-Art Supervised \\Depth and Normal Predictors}

\author{
Yuanyi Zhong, Anand Bhattad, Yu-Xiong Wang, David Forsyth
\vspace{5pt} \\
University of Illinois Urbana-Champaign
\\
{\tt\small \{yuanyiz2, bhattad2, yxw, daf\}@illinois.edu}
}

\maketitle

%%%%%%%%% ABSTRACT
\begin{abstract}
   Dense depth and surface normal predictors should possess the equivariant property to cropping-and-resizing -- cropping the input image should result in cropping the same output image. However, we find that state-of-the-art depth and normal predictors, despite having strong performances, surprisingly do not respect equivariance. The problem exists even when crop-and-resize data augmentation is employed during training. To remedy this, we propose an equivariant regularization technique, consisting of an averaging procedure and a self-consistency loss, to explicitly promote cropping-and-resizing equivariance in depth and normal networks. Our approach can be applied to both CNN and Transformer architectures, does not incur extra cost during testing, and notably improves the supervised and semi-supervised learning performance of dense predictors on Taskonomy tasks. Finally, finetuning with our loss on unlabeled images improves not only equivariance but also accuracy of state-of-the-art depth and normal predictors when evaluated on NYU-v2. \href{https://github.com/mikuhatsune/equivariance}{(GitHub link)}
\end{abstract}

%%%%%%%%% BODY TEXT
\def\gT{{\mathcal{T}}}
\newcommand{\E}{\mathbb{E}}

\section{Introduction}

\begin{figure}[t]
    \centering
    \includegraphics[width=\linewidth]{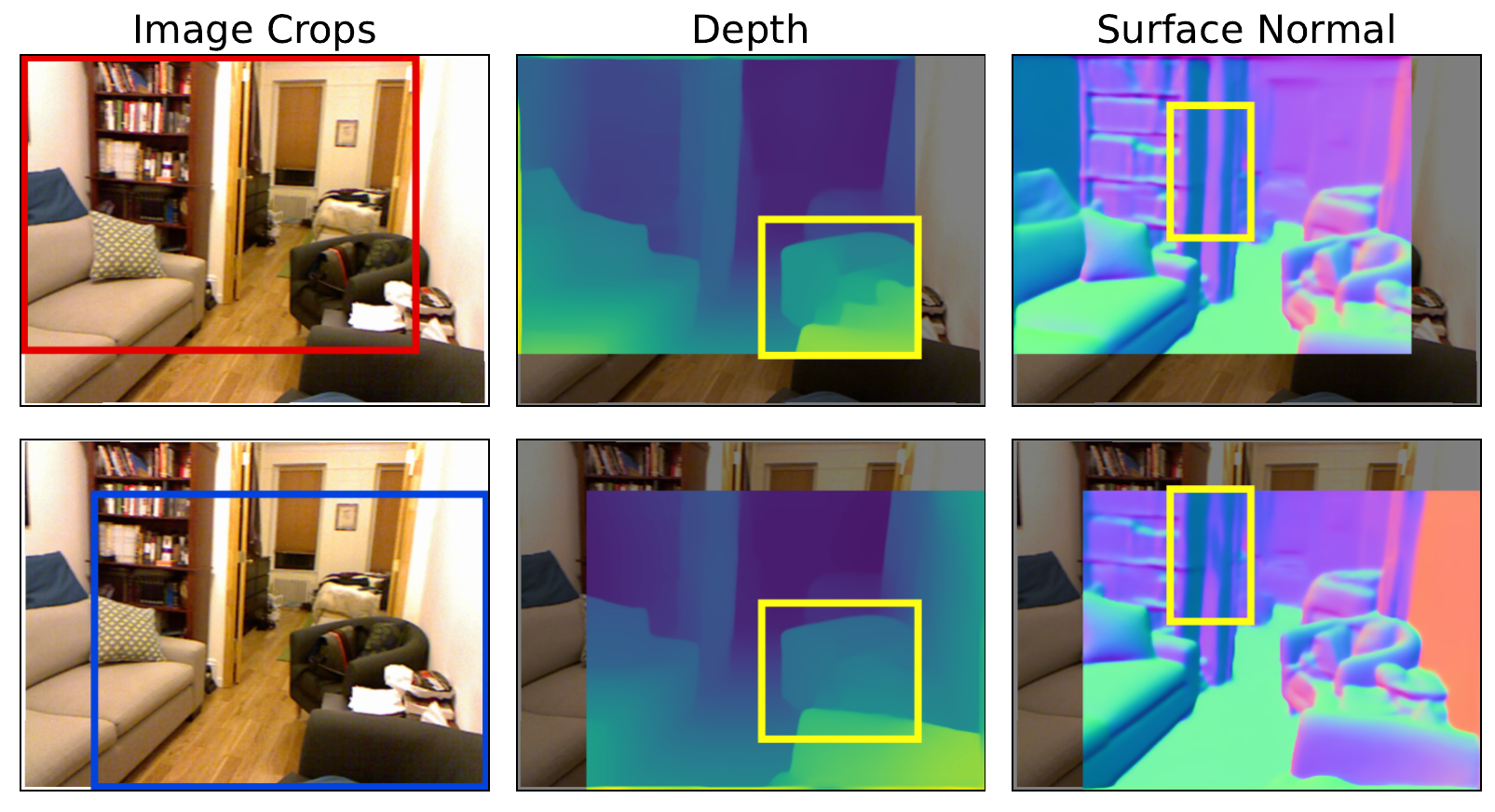}
    \caption{State-of-the-art depth and surface normal predictors fail to capture equivariance, while we know equivariance needs to hold for an ideal depth/normal predictor (when adjusted for prediction scale and offset). We crop and resize two patches (red \& blue) from the same scene, then extract depths/normals with pre-trained models from \cite{ranftl2021vision} (MiDaS-v3) and \cite{bae2021estimating}. We notice clear discrepancies between the predictions of the two crops, as highlighted by the yellow boxes. The same issue exists in other depth predictors (Figure~\ref{fig:depth_issue}) and dense prediction tasks as well (supplementary).}
    % The supplementary material contains a video clip further illustrating this issue.
    \label{fig:sota_bad}
\end{figure}

Depth regression \cite{ranftl2020towards,ranftl2021vision,lee2019big,yin2019enforcing,huynh2020guiding,yang2021transformer,miangoleh2021boosting,facil2019cam,yin2021learning,bhat2022localbins,patil2022p3depth} and surface normal regression \cite{bae2021estimating,huang2019framenet,wang2020vplnet,do2020surface} are image-to-image dense prediction tasks that involve predicting an output image of the same size as the input image. This contrasts with image classification, where only one or a few category labels are predicted per image. 
A shared feature among depth and normal prediction tasks is that they naturally require equivariance, such that a geometric transform (e.g., random cropping) applied to the input image results in the same transform to the output image \cite{lenc2015understanding,cohen2016group,finzi2020generalizing,kondor2018generalization}, when the effect (scale of the depth prediction) of camera intrinsic change due to cropping is accounted for.
This is because the relative depths and normals are derived from the underlying geometrical and physical properties of the scene that are not affected by viewport changes. Consequently, a good depth or normal predictor must have the equivariant property.

To our surprise, we find state-of-the-art well-engineered depth and normal predictors often fail at equivariance. We investigate two recent models: the MiDaS CNN-based (v2.1) and Transformer-based (v3.0) depth predictors from \cite{ranftl2020towards,ranftl2021vision} and the uncertainty-guided CNN-based surface normal predictor from \cite{bae2021estimating}. We generate a pair of resized crops of the same test image from NYU-v2 \cite{silberman2012indoor}, extract predictions with the networks, and measure equivariance by comparing and computing the mean errors between the predictions of the two crops. A more equivariant network would produce smaller errors from this procedure. We discover that the examined depth and surface normal predictors do not handle equivariance to cropping very well, as shown in Figure~\ref{fig:sota_bad}. There are prominent, sometimes structural, inconsistencies in the predictions of the two crops. For this particular scene, the mean error induced by cropping is significant -- as large as 12.6\% absolute relative error (AbsRel) between crops for depth prediction, making it comparable to the overall AbsRel error to ground truths (13.7\%).
Given the widespread use of such dense predictors, for example, MiDaS-v3 for the depth-guided inference in Stable Diffusion-v2 \cite{rombach2022high}, it is imperative to solve such an issue.

Data augmentation is a widely-used strategy to promote the equivariance of models during training. In each mini-batch, instead of seeing the original images, the network sees random resized crops of them. The network is implicitly trained to cope with the variations caused by random crops in a straightforward data-driven manner. However, the problem persists even when randomly resized cropping augmentation is used during training. In fact, the state-of-the-art models we tested, for example, the MiDaS depth networks \cite{ranftl2020towards,ranftl2021vision}, are already trained on random crops. This suggests that augmentation alone is not a sufficient solution to the equivariance issue. 
Other methods to enforce equivariance include invariant inputs and equivariant architectures, but they involve a nontrivial additional effort to construct and do not apply to the cropping transform we are concerned about. Therefore, we compare our approach to the data augmentation strategy as our primary baseline.

In this paper, we propose an equivariant regularization approach built on top of data augmentation to improve equivariance in dense depth and normal prediction networks. Our approach consists of two parts: an equivariant averaging step of the outputs of random crops, and an equivariant loss between the crop outputs and the average output. The averaging step is based on the key observation that the full output average of all possible transforms of a transformation group guarantees equivariance to that group. Our sampling version is effectively an unbiased estimate of the full average. The equivariant loss enforces self-consistency and promotes equivariance explicitly rather than implicitly as in data augmentation. Thanks to the flexible formulation, our approach can be applied to any layer of popular network architectures (e.g., CNN or Vision Transformer \cite{dosovitskiy2020image}), and with unlabeled images -- both are beyond what data augmentation can do. Meanwhile, our approach retains the benefit of data augmentation, as it imposes no extra cost during testing because the network architecture and the inference procedure are not changed in any way.

Empirically, we demonstrate the effectiveness of our equivariant regularization approach in supervised, semi-supervised and unsupervised learning settings. In the supervised setting, we benchmark our approach against the no-augmentation and augmentation baselines on edge detection, depth prediction, and surface normal prediction tasks of the Taskonomy dataset \cite{zamir2018taskonomy}. We find that our approach overcomes the ineffectiveness of using data augmentation alone. In the semi-supervised setting, we show our approach benefits from unlabeled data, improving the sample efficiency further. Finally, in the unsupervised setting, we demonstrate the capability to adapt the state-of-the-art depth and surface normal models to the NYU-v2 dataset \cite{silberman2012indoor} (which these models are not trained on), and improve their accuracy and equivariance, without using any ground truth labels.

To summarize, our contributions are the following:
\begin{itemize}[itemsep=2pt,parsep=0pt,topsep=5pt]
    \item We point out an obvious but overlooked issue: The state-of-the-art depth and normal prediction networks fail at equivariance to cropping.
    \item We propose an equivariant regularization approach to learn more equivariant networks effectively.
    \item We show empirical successes of our approach in a range of settings, and improve the equivariance and accuracy of the state-of-the-art depth and normal models.
\end{itemize}

\section{Related Work}

\parasmall{Equivariance in ML.}
Equivariance is tied closely to geometry and symmetry. The entire subject of physics is founded on concepts surrounding symmetry. A wide range of natural phenomena admits equivariance inherently since the underlying mechanism is oftentimes geometrical. As a consequence, a lot of data that machine learning deals with has the equivariance property. For example, camera photography follows simple 3D geometry rules, thus a shift in camera position leads to a shift in the photograph; the molecules and point clouds have translation and rotation symmetry in 3D, thus an SE(3) transform should not change any property. Therefore, it is natural to consider equivariance in developing machine learning models.

\para{Equivariance in 2D computer vision.}
Convolutional neural networks (CNN) for 2D images are shown to have the approximate translation equivariance property due to the nature of convolution \cite{lenc2015understanding}. Classic antialiasing applied to CNNs improves the shift-equivariance (invariance) by overcoming the signal alias introduced by downsampling layers \cite{zhang2019making}. There is a line of work developing rotation equivariant 2D CNNs \cite{cohen2016group,worrall2017harmonic,marcos2017rotation,weiler2018learning,weiler2019general}. The transformation group for 2D rotation is the Special Orthogonal group SO(2), and the Special Euclidean group SE(2) if the translation is allowed. The derivation of group equivariance constraint typically results in steerable filters constructed from 2D harmonic bases. The convolution filter weights are parameterized as a linear combination of the harmonic bases.

Equivariance can also be achieved by parameter sharing of the neural net weights \cite{ravanbakhsh2017equivariance}. However, this approach is only possible for limited kinds of groups, such as 90-degree rotations.
2D scale equivariant CNN has been studied \cite{marcos2018scale,worrall2019deep,sosnovik2019scale}. This is typically done by applying the same convolution kernel on several scales or constructing steerable filters from the bases. Scale equivariant network design has been applied to 3D object detection to achieve depth equivariance \cite{kumar2022deviant}.
Equivariant network design method can be generalized to other groups
\cite{kondor2018generalization,finzi2020generalizing,romero2020attentive,yarotsky2022universal,murphy2019janossy}
and has rich theory in math and physics \cite{cohen2019general,he2021gauge,cohen2019gauge}. Equivariance can also be achieved by transforming the data to canonical coordinate systems \cite{tai2019equivariant,puny2021frame,gandikota2022simple}.
In particular, \cite{puny2021frame} transforms the data to key canonical frames of the group and averages over those frames, while we average over a random sample of the cropping transform.
Transformers are the current state-of-the-art neural net architecture \cite{dosovitskiy2020image}. People have sought to combine Transformer and equivariance, resulting in 
Lie-Transformer \cite{hutchinson2021lietransformer}.

In terms of applications, there is good evidence that equivariance benefits image semantic segmentation \cite{cho2021picie,subhani2020learning,melas2021pixmatch}, object detection to shifting \cite{manfredi2020shift} and rotation of images \cite{han2021redet}. Equivariance is also useful for generative modeling, for example, for normalizing flow-based generative models \cite{kohler2020equivariant,satorras2021n}, and variational autoencoders \cite{keller2021topographic}. Equivariance to rotation is beneficial in digital pathology \cite{veeling2018rotation}. Extension to time-equivariance for video is also possible \cite{jenni2021time}.

\para{Equivariance in self-supervised learning.}
Equivariance and invariance are useful in self-supervised learning. The popular contrastive learning algorithm relies on the invariance of representations between augmented views of the same image \cite{chen2020simple,he2020momentum,caron2021emerging,grill2020bootstrap,chen2021exploring}. More recently, people are exploring ways to use equivariance in contrastive learning \cite{xie2022should}. Leveraging equivariance to cropping transform results in dense contrastive learning at pixel-level: for example, PixelPro \cite{xie2021propagate} and DenseCL \cite{wang2021dense} for pre-training, PC2Seg \cite{zhong2021pixel} for semi-supervised semantic segmentation; and at region-level: RegionCL \cite{xu2022regioncl}, DetCon \cite{henaff2021efficient}. Equivariance to 4-way rotation can be jointly used with the image-level contrastive objective to improve performance \cite{dangovski2021equivariant}. Self-supervised learning from equivariance between flow transformations of the input image is also effective \cite{xiong2021self} and between matching points for landmark representation learning \cite{thewlis2019unsupervised}.

These works are especially successful for downstream segmentation and detection tasks. However, the advancements in these work have yet to be thoroughly explored in the state-of-the-art depth or normal predictors to the best of our knowledge \cite{ranftl2020towards,ranftl2021vision,bae2021estimating}, where the dominant paradigm is still supervised training. Inspired by prior work in SSL and segmentation, our work brings in the powerful idea of equivariance to improve state-of-the-art supervised depth and normal predictors.

\section{Background}

We give some background on the issue of equivariance and how people typically approach equivariance in the literature.

\begin{definition}[Equivariance]\label{def:equi}
Formally, a function $f: X \rightarrow Y$ is equivariant under the action of a group $G$ on $X$ and a group of $G'$ on $Y$ if for any $t \in G$ there exists $t' \in G'$ such that $f \circ t(x) = t' \circ f(x)$. More commonly, it is true that $G = G'$, i.e., the transformation on both $X$ and $Y$ domains is the same, and the condition becomes
\begin{equation}
f \circ t(x) = t \circ f(x) .
\end{equation}
It essentially states that transform $t$ commutes with $f$ and changes the input and output in the same way.

Invariance can be regarded as a special case of equivariance where $g'$ is always the identity operation. In other words, invariance means $f \circ t(x) = f(x)$ for any action $t \in G$. For example, equivariance is useful for modeling transform-aware phenomena, while invariance is useful for modeling classification tasks.
\end{definition}

% \para{Deep nets have limited equivariance.}
\para{Non-equivariance issue in depth and normal predictors.}
Convolutional neural networks possess a certain degree of translation equivariance, but for a broader class of transformations, such as resized cropping, rotation, and scaling, they are not designed to capture equivariance. More recent networks such as Transformers \cite{dosovitskiy2020image} have little inductive biases built-in, they likely do not possess much equivariance on their own as well, and need to see a large number of training examples to learn equivariance in a purely data-driven manner.
% Such a drawback of deep nets has motivated researchers to study ways to build equivariance into the model.

\begin{figure}[t]
\scriptsize
\centering
\hspace{-0.85mm}\includegraphics[width=1.01\linewidth]{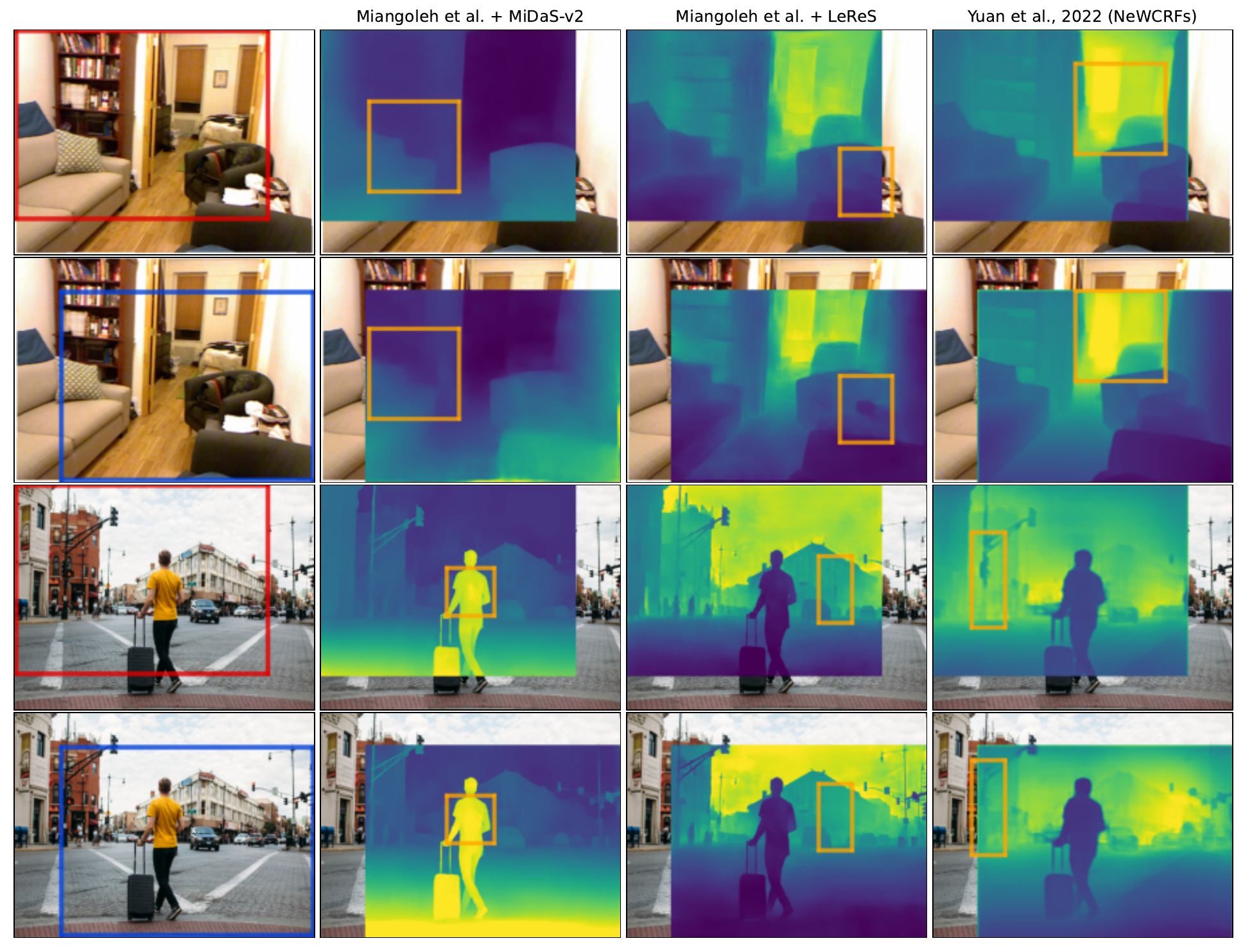}
~\vspace{-3.5mm}
\caption{Equivariance failures in more depth predictors \cite{miangoleh2021boosting,yuan2022neural} than Figure~\ref{fig:sota_bad}, suggesting the issue is \emph{prevalent}. The prediction values of the two crops are aligned with the least square. The 2nd column is disparity, others are depth. Top-down, left-right: Notice the blurry/sharp edge, missing object on the stand, wall, pattern on the person's back, vertical line on the building, and inconsistent traffic light.}
\label{fig:depth_issue}
\vspace{-5mm}
\end{figure}

Figure~\ref{fig:sota_bad} shows the failure of equivariance of depth \cite{ranftl2021vision} and surface normal predictors \cite{bae2021estimating}. The issue is not unique to these two methods. We examined two more recent approaches \cite{miangoleh2021boosting,yuan2022neural} in Figure~\ref{fig:depth_issue}. \cite{miangoleh2021boosting} is especially interesting, because it similarly merges small crops to reduce the error of pre-trained depth predictors at inference time.

While \cite{miangoleh2021boosting} and we both use the average idea, we conduct averaging at training time instead of inference time. The fact there are still structural changes with cropping suggests that inference-time 
averaging does not completely solve the issue. On the other hand, a network trained with our approach improves equivariance without extra inference costs. Additionally, we focus on reducing inconsistent predictions between (often large) crops, while \cite{miangoleh2021boosting} focuses on improving depth details with lots of \emph{small} crops, not necessarily improving equivariance.

Another point to note is that camera intrinsic change (of center and scale) caused by random cropping cannot explain the discrepancies in Figures~\ref{fig:sota_bad}, \ref{fig:depth_issue}.
Camera intrinsic change may lead to overall shifting or rescaling of predicted depths, as noticed and fixed by \cite{facil2019cam}. However, the failures we observe are structural and related to the content, like missing or creating non-existent objects, even with large crops that only mildly affect intrinsics. What we observe is a separate non-equivariance issue that needs to be solved.

\subsection{Existing Approaches}

We recognize three types of methods to introduce equivariance into machine learning models. They have different advantages, disadvantages, and suitable application domains. Unfortunately, data augmentation has been the only approach that works for the random resized cropping transform in the dense prediction tasks we study here, which is still inadequate.

\smallskip\para{Data Augmentation.}
Data augmentation is the simplest way to encourage the equivariance of ML models \cite{chen2020simple,he2020momentum,yang2019invariance}. As long as the transformation function is available, we can artificially create more training examples by transforming the original data randomly. In the case of invariance, we only augment the input data, e.g., the input images for image classification, where the output of the machine learning model is trained to be invariant to the transformation. In the case of equivariance, we can augment the input and the output simultaneously, e.g., the RGB images and depth maps. Commonly used data augmentations include random color jittering, random resizing, and random cropping. The benefit of this approach is simplicity. One can keep the training pipeline and the modeling part the same. However, the downside is potential inefficiency, as we also see with state-of-the-art depth and normal networks in Figure~\ref{fig:sota_bad}. The model may need to see a very large quantity of augmented data examples to learn the equivariance property in a data-driven manner.

\smallskip\para{Invariant Inputs.}
The second type of approach converts the data into a format that is invariant or equivariant to the specific transformation. An example of this approach is the distance matrix when dealing with molecular data \cite{satorras2021n,fuchs2020se}. People turn the Cartesian coordinates of points (atoms) into a relative distance matrix between pairs of points. It is easy to verify that the distance matrix is invariant to 3D translation and rotation. If the model only depends on the invariant inputs, it is guaranteed to be equivariant or invariant to any input transformation. Another example is the alignment procedure in 3D point cloud/data processing \cite{thomas2018tensor,chen2021equivariant,li2019discrete,li2021closer,sajnani2022condor}, where one can align the points according to their principle canonical axes either globally or locally. This approach works well when the invariant inputs exist, contain sufficient information for the task, and are easy to compute. However, the usage is limited when these requirements are not met. For example, it is not immediately clear how to come up with invariant inputs for standard image augmentations including the crop-and-resize in dense prediction tasks.

\smallskip\para{Equivariant Architecture.}
A rich line of research focuses on building equivariance property into the ML model in a ``hard-wired'' manner \cite{cohen2016group,worrall2017harmonic,marcos2017rotation,weiler2018learning,weiler2019general,kondor2018generalization,finzi2020generalizing,romero2020attentive,cohen2019general,he2021gauge,cohen2019gauge}. 
They typically start from a group theory and symmetry standpoint and derive functional forms that satisfy equivariance (relatively) precisely with math and physical science flavor. For example, 2D convolution can be derived for the planer translation group with a Fourier basis. Mirroring constraints on convolutional kernels can be derived for the left-right mirroring group. Convolutions with spherical harmonics can be derived for SO(3) groups. The advantage of this type of approach is that it is principled, exact, and sample-efficient. As rewriting the functional form with equivariance in mind restricts the size of the function class and introduces strict inductive biases, searching for the right hypothesis from data may become easier, and the learning may be accelerated. However, the disadvantages are that one has to modify the model architecture, and deriving the analytical solution for the equivariance basis might be complicated or even impossible, such as for the randomly resized transform in our dense prediction case.

\section{Our Approach}

Our approach is equivariant regularization. Equivariant property can be imposed by a regularization loss in a ``soft'' manner together with data augmentation.

We will first describe the mathematical intuition of our approach. We start with the definition of equivariance, then introduce the equivariant average operator as a core technique. The average operator has nice properties, such as being able to turn a non-equivariant function into an equivariant one. We leverage such properties to build our equivariant regularization technique. We introduce a differentiable equivariant loss between the average and individual predictions, which can be minimized to encourage equivariance.

Now we consider the following average operator.

\begin{definition}[Equivariant average operator]\label{def:ave}
Let $P(t)$ be a uniform distribution over group elements $t \in \gT$. We define the equivariant average of an arbitrary function $f$ as
\begin{equation}
    \bar{f}(x) = \E_{t\sim P(t)} \left[ t^{-1} \circ f \circ t(x) \right] .
\end{equation}
\end{definition}

The intuition behind this definition is variance reduction. Each summand in the expectation is an estimator of the predicted quantity, with some variance. Taking crop transform as an example, each $t$ takes a particular cropped view of the input image, $f$ makes the predictions for this view, and $t^{-1}$ transforms the predicted image back to the original coordinate frame. Now, each $t$ might lead to a different type of error in the prediction, but averaging (or summing) over all of them will make the differences disappear.
This intuition is formally described in the following properties.

\begin{proposition}\label{prop:ave}
The averaged $\bar{f}$ is equivariant to $\gT$.
\end{proposition}

\begin{proof}
For any $t \in \gT$, it is straightforward to verify that
\begin{align}
\bar{f} \circ t(x)
&=
\E_{t_1\sim P(t)} \left[ t_1^{-1} \circ f \circ t_1 \circ t (x) \right]  \tag{definition of $\bar{f}$}
\\&=
\E_{t_2\sim P(t)} \left[ (t_2 \circ t^{-1})^{-1} \circ f \circ t_2 (x) \right]  \tag{let $t_2 = t_1 \circ t$, associativity}
\\&=
\E_{t_2\sim P(t)} \left[ t \circ t_2^{-1} \circ f \circ t_2 (x) \right]
\\&=
t \circ \E_{t_2\sim P(t)} \left[ t_2^{-1} \circ f \circ t_2 (x) \right]  \tag{linearity of expectation}
\\&=
t \circ \bar{f} (x)  \tag{definition of $\bar{f}$}
\end{align}
which is the definition of equivariance.
\end{proof}

\begin{proposition}\label{prop:idem}
The equivariant average operator preserves the function $f$ if $f$ is already equivariant. As a corollary, the operator is idempotent, namely, $\bar{\bar{f}} = \bar{f}$.
\end{proposition}

\begin{proof}
Use the equivariance definition of $f$ and the associativity of function composition,
\begin{equation}
\begin{aligned}
\bar{f} (x)
&=
\E_{t\sim P(t)} \left[ t^{-1} \circ f \circ t (x) \right]
=
\E_{t\sim P(t)} \left[ t^{-1} \circ t \circ f (x) \right] \\
&=
\E_{t\sim P(t)} \left[ f (x) \right]
=
f (x) .
\end{aligned}
\end{equation}
From Proposition~\ref{prop:ave}, we know that $\bar{f}$ is always an equivariant function, therefore $\bar{\bar{f}} = \bar{f}$.
\end{proof}

Proposition~\ref{prop:ave} and \ref{prop:idem} are practically useful. They together justify treating the equivariant average as a normalization operation because (1) it can turn an arbitrary non-equivariant function into an equivariant one, (2) applying it twice has no further effect than applying it only once.

Once we have the equivariant average, we can use it as a training target to achieve higher equivariance. Specifically, we construct the following loss function based on the equivariant average operator to encourage equivariant property on a trainable function $f$. This $f$ can be the output of a dense prediction network or any intermediate features.

\begin{definition}[Equivariant loss]\label{def:equi_loss}
We define the Equivariant loss as the mean L2 error between the individual prediction $f \circ t(x)$ and the averaged prediction $\bar{f}(x)$:
\begin{equation}
\xi(f) = \frac1{Z(\bar{f})} \E_{t\sim P(t)} \left[ \| f \circ t(x) - t \circ \bar{f}(x) \|_2^2 \right]
\end{equation}
where $Z$ is the normalizing constant:
$
Z(\bar{f}) = \| \bar{f}(x) \|_2^2
$
assuming $\bar{f}$ is not everywhere $0$.
\end{definition}

The normalizing constant $Z$ is a technical trick to normalize the scale of the equivariant loss. Without $Z$, simply multiplying $f$ with a scalar will enlarge the equivariant loss, which is undesired. With $Z(\bar{f})$, since $Z(\alpha \bar{f}) = \alpha^2 Z(\bar{f})$, 
we can show that
\begin{equation}
\xi(\alpha f) = \frac1{\alpha^2 Z(\bar{f})} \E_{P(t)} \left[ \alpha^2 \| f \circ t(x) - t \circ \bar{f}(x) \|_2^2 \right] = \xi(f) .
\end{equation}

\begin{figure}[t]
    \centering
    \includegraphics[width=0.91\linewidth]{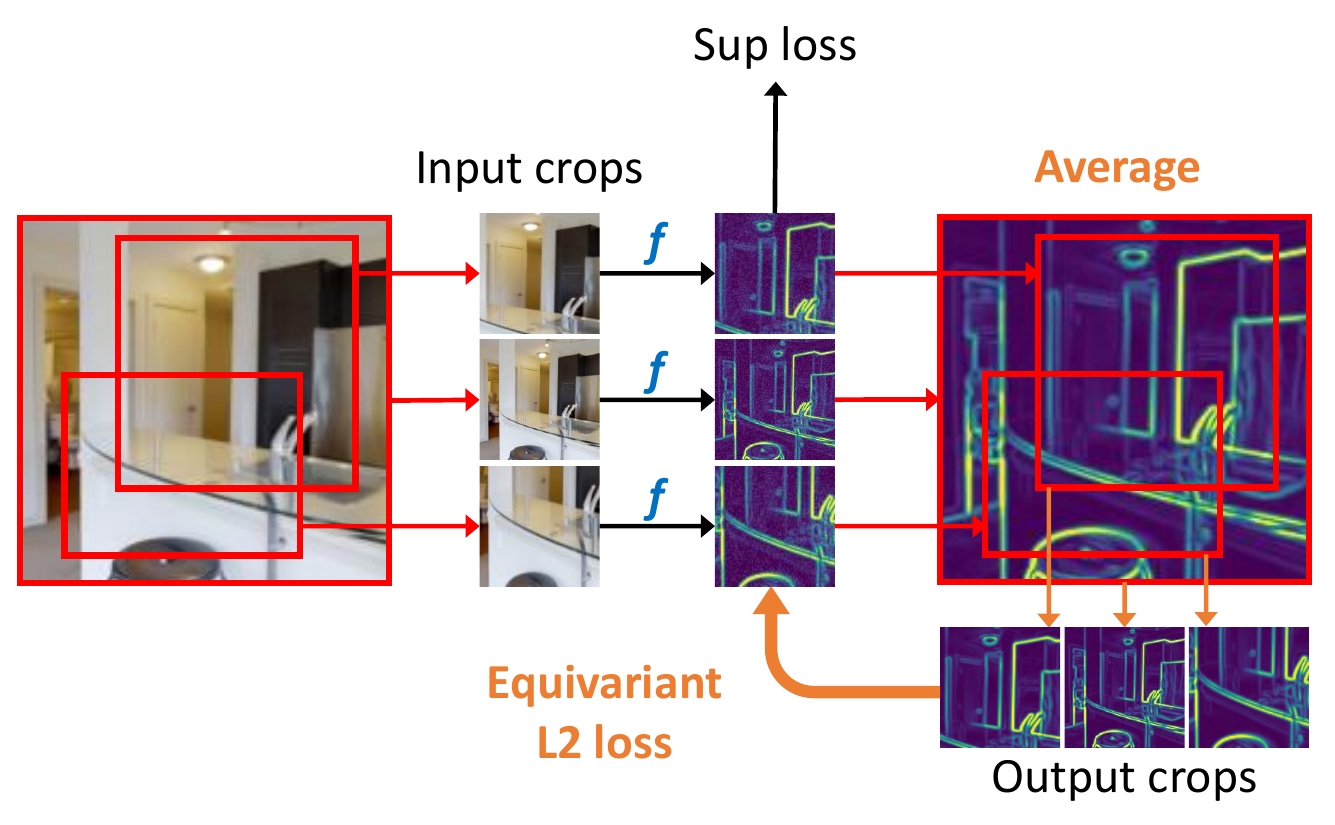}
    \caption{Illustration of our equivariant regularization approach with 3 crops for the 2D texture edge detection task. We generate 3 crops of the input image and pass them through the network $f$ to get 3 outputs. We perform the equivariant average to register them together and get the averaged output. Next, we crop the averaged output to obtain 3 output crops. They correspond to the same image regions as the input crops. We use them as training targets (gradient-stopped) for the individual crop's outputs. A standard supervised loss would use the ground truth as a target, whereas our approach uses the averaged output as a target.}
    \label{fig:equi}
\end{figure}

In practice, it is often computationally infeasible to enumerate and average over all possible transforms $t$'s, as there are too many. This is the case for the commonly-used random resized cropping augmentation we care about. The random cropping induces a combination of {\em continuous} rigid transformation and scaling groups. We can circumvent this issue by Monte Carlo estimation, i.e., sample a couple of $t$'s and compute the empirical average of $\bar{f}$ and loss. The sample size $K$ is a hyper-parameter to be studied empirically that trades off accuracy and computation efficiency. The following equations state the sampling version:
\begin{align}
\label{eq:equi_loss}
\widehat{\bar{f}(x)} &= \frac1K \sum_{k} \left[ t_k^{-1} \circ f \circ t_k(x) \right] ,\\
\widehat{\xi(f)} &= \frac1{Z(\bar{f})} \frac1K \sum_{k} \left[ \| f \circ t_k(x) - t_k \circ \widehat{\bar{f}(x)} \|_2^2 \right] .
\end{align}

We can attach the equivariant loss onto any layer of a neural net, and train the network with the linear combination of the task loss and the equivariant loss. 
Formally, assume the task loss is $\ell$ and the equivariant loss is imposed on the $l$-th layer $f_l$ with loss coefficient $\lambda_l$, the total loss writes as
\begin{equation}
    \ell_{\text{total}}(f) = \ell(f(x), y) + \lambda_l \widehat{ \xi(f_l) } .
\end{equation}

Figure~\ref{fig:equi} illustrates our equivariant regularization approach regarding the random resized crop transform.

\subsection{Discussion}

The difference between our approach and the equivariant architecture is that we do not emphasize exact equivariance in this case. Once trained, the model is allowed to have a certain degree of non-equivariance than a strict equivariant model but is expected to possess a higher degree of equivariance than a baseline model without any special equivariance treatment.

Our approach builds on top of data augmentation. It strikes a balance between the equivariant architecture and the data augmentation approaches. It improves upon pure data augmentation by introducing explicit learning signals for equivariance and does not require the complicated derivation or architecture modification of a strict equivariant model. In fact, there should be no additional overhead to a regular model at inference time. We also have the flexibility to adjust the regularization strength by tuning loss coefficients, when the task is not perfectly equivariant or to strive for better overall performance. Our approach can also extend naturally with unlabeled data since the procedure does not involve ground truth labels.

\section{Experiment}

\subsection{Datasets, Models, and Tasks}

We evaluate our approach on three data labeling settings with increasing difficulty: supervised, semi-supervised and unsupervised.

\para{Supervised setting.}
For a supervised setting, we use the Taskonomy dataset \cite{zamir2018taskonomy} standard Tiny splits for experimentation. Taskonomy contains RGB-D scans of indoor building scenes. Several dense prediction tasks are derived from the scans. We focus on 2D texture detection, a low-level vision task; and depth z-buffer prediction and surface normal prediction, two related geometric vision tasks. The Tiny split has 24 training buildings (250K images; originally 25 buildings, 1 building is removed due to data corruption) and 5 validation buildings (52K images).

The model involved in the supervised setting is the standard U-Net from XTaskConsistency \cite{zamir2020robust}. This U-Net has 6 downsampling, 6 upsampling blocks, and the skip connections between corresponding downsampling and upsampling stages. The supervised loss is the L1 loss between the outputs and ground truth targets. For depth, we use inverse depth (i.e., disparity) following \cite{ranftl2020towards}. We apply our equivariant regularization loss technique on the second to last convolutional layer of the network. The loss location will be ablated. The loss coefficient is set to 1e-4. For each image, we generate $K=3$ random crops with scale variation uniformly sampled from 0.4-1.0, aspect ratio from 3/4-4/3, allowing at most 20\% padding length, and common color jittering (brightness = contrast = saturation = 0.4, hue = 0.1). In practice, we also employ a weighting window with smooth edges when computing the equivariant average to suppress the boundary effects. We train all models with the AdamW optimizer \cite{loshchilov2017decoupled}, with learning rate cosine annealed from 1e-3 to 0, and weight decay 1e-4, for 78K gradient steps with batch size 32 distributed on 4 GPUs. Input resolution is 256x256. 
To maintain \emph{fair comparison}, the supervised baseline is also trained with $K=3$ crops per image, therefore the wall-clock time of all experiments is roughly the same.

The standard evaluation metrics are L1 error for edge; the percentage of pixels with a relative depth error larger than 1.25 ($\delta\!\!>\!\!1.25$), mean absolute relative error (AbsRel) for depth; and mean angular error for surface normal \cite{zamir2018taskonomy,zamir2020robust}. Since depth predictor is usually not aligned to metric depth, i.e., they output arbitrary scale, we align predicted depth to ground truth with least square regression following MiDaS \cite{ranftl2021vision,ranftl2020towards}. The detail is described in Supp. C.

The essential question we want to study is whether our approach performs better than the usual data augmentation approach in achieving equivariance and accuracy.

\para{Semi-supervised setting.}
For this, we concentrate on the depth prediction task. We use 6 or 12 buildings out of 24 buildings in the training set as the labeled portion, and use the rest of the buildings as the additional unlabeled data. The model, hyper-parameters, and optimization schedule are the same as above. During training, we sample two equal-sized mini-batches (2 $\times$ 32 images $\times$ 3 crops) from the labeled and unlabeled data streams, respectively. We impose the supervised loss only on the labeled batch and our equivariant loss on both batches.

This setting is to test whether our approach provides additional benefits from unlabeled data, which is not possible with the simple data augmentation approach.

\para{Unsupervised setting.}
We focus on unsupervised finetuning of pre-trained state-of-the-art models on the NYU-v2 dataset \cite{silberman2012indoor}. The NYU-v2 dataset contains RGB-D scans of 464 indoor scenes, of which 249 scenes (795 images) are used for training and 215 scenes (654 images) for testing. The resolution is 480x640.

We consider the MiDaS-v2.1 and v3.0 depth predictors from \cite{ranftl2020towards,ranftl2021vision} with CNN and Vision Transformer backbones, respectively. According to their paper, these models are trained on random augmented crops of length 384; therefore, we set the input shape as 384x288 in our unsupervised finetuning experiments. We also consider the pre-trained uncertainty-guided surface normal predictor from \cite{bae2021estimating}. This network is based on the convolutional EfficientNet backbone \cite{tan2019efficientnet}. We set the input shape as 640x480 for surface normal to match their training setting. We use AdamW optimizer for 800 steps, with a small learning rate of 1e-5 for depth and 1e-4 for surface normal, as we find them work the best. Two loss functions are involved in finetuning: the first is the supervised loss between the outputs and the pseudo labels generated from the pre-trained checkpoints, and the second is our equivariant loss on the output of the network. We set the equivariant loss coefficient to 1e-4 as well. We sample $K=3$ random crops per image with scale variation 0.7-1.0, at most 10\% padding and common color jittering.

Note that all the pre-trained checkpoints investigated here are {\em not} trained on NYU-v2. We want to see if our approach can boost the performance of state-of-the-art pre-trained models on this new dataset, by encouraging equivariance alone, {\em without} using any ground truth labels.

\subsection{Results}

\begin{table}[t]
    \small
    \centering
    \caption{Supervised setting: Taskonomy Edge2D, Surface Normal, and Depth-ZBuffer. Equivariant regularization on U-Net improves validation performance. Sup baseline refers to baseline without data augmentation, Aug refers to with augmentation, EqLoss refers to our equivariant loss approach. Ang error is the mean angular error in degrees, $\delta\!\!>\!\!1.25$ is the percentage of pixels with a relative depth error larger than 1.25.}
    \label{tab:res_task}
    \setlength{\tabcolsep}{2.3pt}
    \begin{tabular}{@{}lllll@{}}
    \toprule
    {\bf Task}    & {\bf Edge2D} & {\bf Normal} & {\bf Depth-Z} \\
    Metric & L1 error ({\scriptsize $\!\times\!10^{-3}$})$\downarrow$   & Ang error ($^\circ$)$\downarrow$ & $\delta\!\!>\!\!1.25$ (\%)$\downarrow$  \\
    \midrule
    Sup baseline & 8.14 & 6.72 & 27.8 \\
    + Aug &  7.35  (-9.7\%) & 6.55  (-2.5\%) & 27.0  (-2.9\%)\\
    + EqLoss (ours) & \textbf{6.35  (-22\%)} & \textbf{6.47  (-3.7\%)} & \textbf{25.0  (-10\%)} \\
    \bottomrule
    \end{tabular}
\end{table}

\begin{table}[t]
    \small
    \centering
    \caption{Semi-supervised setting: Taskonomy Depth-ZBuffer. Equivariant regularization with additional unlabeled data improves more. We treat 1/4, 1/2 of the original data as labeled images; the rest as unlabeled images. Sup + EqLoss refers to using only the labeled part and our loss. Semi-sup + EqLoss refers to applying our equivariant loss on both the labeled and unlabeled images. $\delta\!\!>\!\!1.25$ is the percentage of pixels with a relative depth error larger than 1.25, AbsRel is the mean absolute relative error.}
    \label{tab:res_semi}
    \setlength{\tabcolsep}{8pt}
    \begin{tabular}{@{}lllll@{}}
    \toprule
        {\bf Labeled portion} & {\bf 1/4}  & {\bf 1/2} & {\bf All} \\
        \#Buildings  & 6      & 12      & 24 \\
        \#Images     & 58,783 & 123,496 & 248,148 \\
    \midrule
        & \multicolumn{3}{c}{$\delta\!\!>\!\!1.25$ (\%)$\downarrow$} \\
        Sup + Aug                & 43.4     & 30.4     & 27.0  \\
        Sup + EqLoss (ours)      & 42.0     & 29.8   & \bf 25.0  \\
        Semi-sup + EqLoss (ours) & \bf 41.0 & \bf 29.3 & \bf 25.0  \\
    \midrule 
        & \multicolumn{3}{c}{AbsRel (\%)$\downarrow$} \\
        Sup + Aug                & 25.2     & 20.3     & 18.9  \\
        Sup + EqLoss (ours)      & 24.9     & 20.2     & \bf 18.0 \\
        Semi-sup + EqLoss (ours) & \bf 24.8 & \bf 19.6 & \bf 18.0  \\
    \bottomrule
    \end{tabular}
    \vspace{-10pt}
\end{table}

\parasmall{Equivariant regularization improves edge, depth, and normal dense prediction tasks in the supervised setting.}
The results are organized in Table~\ref{tab:res_task}. Comparing the first row to the second, we confirm that data augmentation is better than no data augmentation, which is known widely. This suggests that the implicit encouragement of equivariance from augmentation is helpful. Comparing the second row to the third, we find that our approach brings noticeable gains on top of data augmentation. We achieve as large as 22\%, 3.7\%, and 10\% error reduction for the edge, normal, and depth predictions relative to the supervised baseline without augmentation in their respective metrics. The results indicate that our approach is a more effective way to enforce equivariance during training than data augmentation and that by doing so, the accuracy is also improved.

\begin{table}[t]
    \small
    \centering
    \caption{Unsupervised setting: Adaptation of state-of-the-art pre-trained depth networks to NYU-v2. We finetune the network with images in NYU-v2, pseudo labels and our equivariant loss (EqLoss row), but without ground truth labels. EqLoss reduces validation errors, while also reducing the validation equivariant loss (EqLoss column), suggesting the network becomes more equivariant. The results compare favorably to other recent methods dedicated to NYU-v2.}
    \label{tab:res_nyu_depth}
    \setlength{\tabcolsep}{2pt}
    \begin{tabular}{@{}lccc@{}}
    \toprule
        Model             & $\delta\!\!>\!\!1.25$(\%)$\downarrow$         & AbsRel(\%)$\downarrow$  & EqLoss$\downarrow$ \\
    \midrule
        \multicolumn{4}{c}{Models trained only on NYU-v2} \\
        Big-to-Small \cite{lee2019big}        & 11.0 & 11.5 & - \\
        Yin et al. \cite{yin2019enforcing}    & 10.8 & 12.5 & - \\
        Huynh et al. \cite{huynh2020guiding}  & 10.8 & 11.8 & - \\
        TransDepth \cite{yang2021transformer} & 10.6 & 10.0 & - \\
    \midrule
        \multicolumn{4}{c}{Models trained on mix datasets transfer to NYU-v2} \\
        MiDaS-2.1 CNN \cite{ranftl2020towards}   & 8.71          & 9.68          & 7.10e-3 \\
        MiDaS-2.1 CNN + EqLoss   & \textbf{7.82} & \textbf{8.92} & \textbf{3.77e-3} \\
        MiDaS-3.0 DPT \cite{ranftl2021vision}    & 8.32          & 9.16          & 7.86e-3 \\
        MiDaS-3.0 DPT + EqLoss  & \textbf{7.75} & \textbf{8.91} & \textbf{3.04e-3} \\
    \bottomrule
    \end{tabular}
\vspace{8pt}
    \centering
    \caption{Unsupervised setting: Adaptation of a state-of-the-art pre-trained surface normal network to NYU-v2. Our unsupervised equivariant finetuning strategy (EqLoss row) reduces the validation mean and median angular errors while reducing the validation equivariant loss (EqLoss column). $11.25^{\circ}$ refers to the percentage of pixels with an error larger than $11.25^{\circ}$. All other models here are trained on ScanNet \cite{dai2017scannet} and evaluated on NYU-v2 directly.}
    \label{tab:res_nyu_normal}
    \setlength{\tabcolsep}{5pt}
    \begin{tabular}{@{}lcccc@{}}
    \toprule
        Model              & Mean$^{\circ}\!\!\downarrow$  & Median$^{\circ}\!\!\downarrow$ & $\!\!11.25^{\circ}\!\!\uparrow$ & EqLoss$\downarrow$ \\
    \midrule
        FrameNet \cite{huang2019framenet} & 18.6 & 11.0 & 50.7 & - \\
        VPLNet \cite{wang2020vplnet}      & 18.0 & 9.8  & 54.3 & - \\
        TiltedSN \cite{do2020surface}     & 16.1 & \bf 8.1 & \bf 59.8 & - \\
    \midrule
        Bae et al. \cite{bae2021estimating}          & 16.03 & 8.47 & 58.8 & 1.26e-2 \\
        
        Bae et al. + EqLoss & \textbf{15.71} & \textbf{8.30} & \textbf{59.4} & \textbf{8.80e-3} \\
    \bottomrule
    \end{tabular}
    \vspace{-10pt}
\end{table}

\para{Equivariant regularization enables unlabeled data in the semi-supervised setting.}
Our equivariant regularization approach naturally extends to the semi-supervised learning setting, where the model can learn from additional unlabeled scene images. Apart from the usually labeled data stream, we train with an additional equivariant loss on the unlabeled data stream. In Table~\ref{tab:res_semi}, we show that, although Sup + EqLoss already brings decent improvements, Semi-sup + EqLoss yields more improvements. These results demonstrate the capability of our approach to leverage unlabeled data to achieve higher label efficiency, which is impossible with the standard augmentation approach.

\para{Unsupervised equivariant finetuning improves state-of-the-art depth and surface normal predictors.}
Another advantage (and important application) of our equivariant regularization approach over data augmentation is the ability to perform an unsupervised finetuning of state-of-the-art dense predictors to downstream datasets without any ground truth labels. Recall that the SoTA dense prediction networks do not preserve equivariance very well as in Figure~\ref{fig:sota_bad}. In this part, we demonstrate improvements in the equivariance and accuracy of the state-of-the-art MiDaS-v2.1 (CNN-based model), v3.0 (DPT-Large, Dense Prediction Transformer) depth predictors \cite{ranftl2020towards,ranftl2021vision} in Table~\ref{tab:res_nyu_depth}, and the uncertainty-guided normal predictor \cite{bae2021estimating} on the NYU-v2 dataset \cite{silberman2012indoor} in Table~\ref{tab:res_nyu_normal}. Note that none of the finetuned models have seen any NYU-v2 ground truth. Our approach consistently boosts their performance of them.

Quantitatively, we observe that not only the accuracy metrics of the depth and normal predictors are increased, but also the EqLoss column in Table~\ref{tab:res_nyu_depth} and \ref{tab:res_nyu_normal}, which measures the equivariant loss on validation images, is reduced by our approach in both cases. Reducing the EqLoss means that the expected error magnitude of prediction inconsistency coming from different crops of the same image is reduced. This suggests improvements of the equivariance of these predictors. A more thorough and direct evaluation of the finetuned predictors is in the supplementary material.

Qualitatively, our equivariant finetuning approach significantly alleviates the non-equivariant issue of state-of-the-art models. Figure~\ref{fig:vis_sota} visualizes the predictions before and after finetuning on NYU-v2. We can clearly see that the inconsistency between predictions of two crops is lessened after finetuning.

\begin{figure*}[tb]
    \centering
    \includegraphics[width=\textwidth]{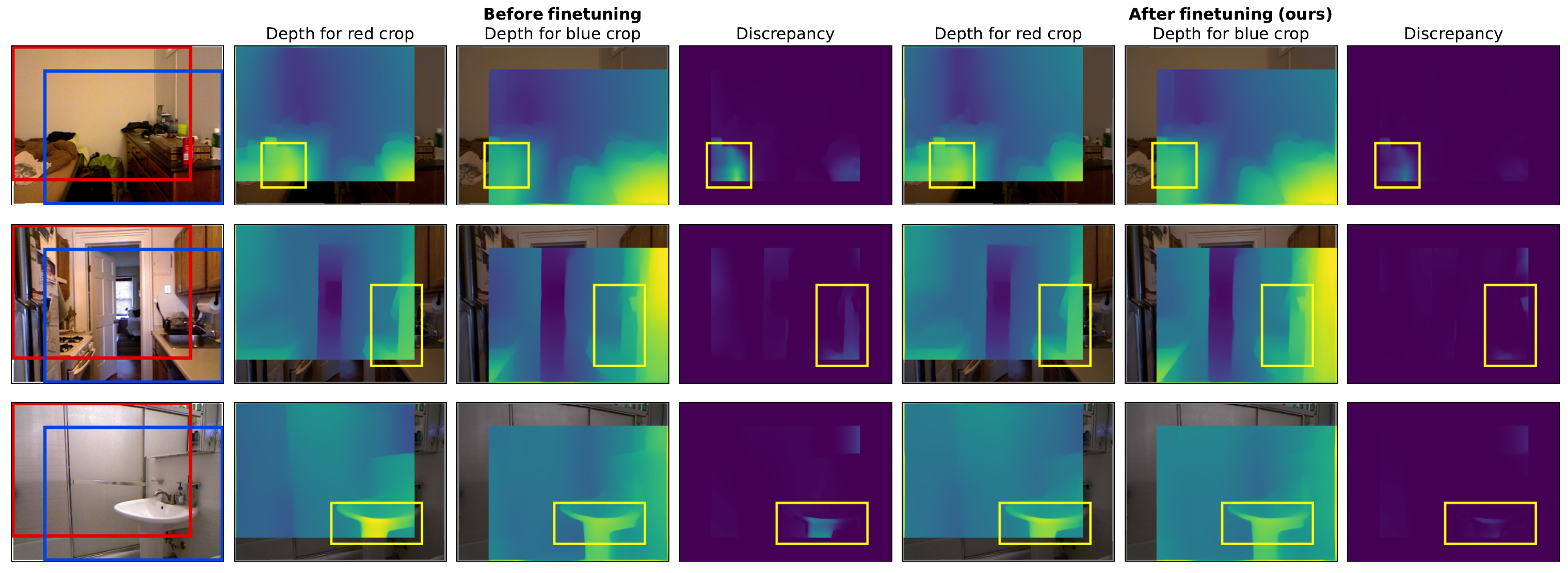}
    \includegraphics[width=\textwidth]{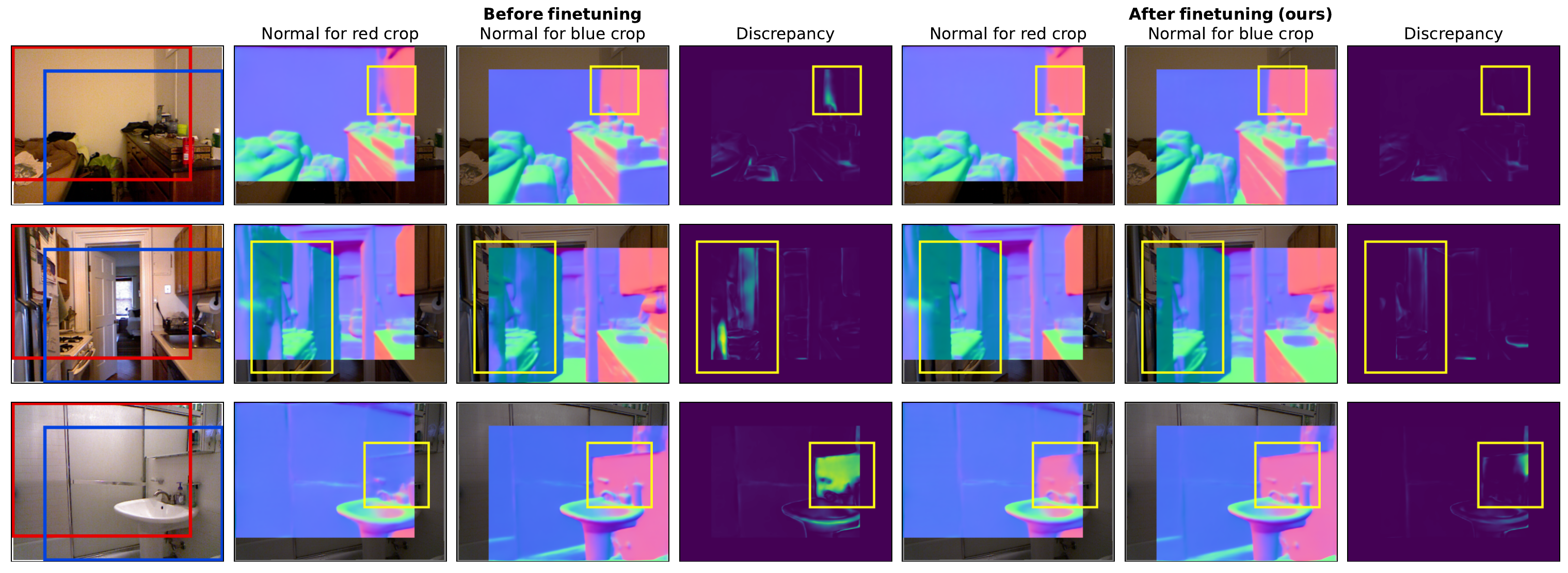}
    \vspace{-12pt}
    \caption{Visualization of predictions of a state-of-the-art depth network (MiDaS-v3.0 DPT Large \cite{ranftl2021vision}) and a surface normal network \cite{bae2021estimating} on two crops (red and blue) of the same image, {\em before and after our unsupervised equivariant finetuning}. The yellow boxes highlight the regions where the pre-trained models struggle with. The discrepancies in those regions are suppressed considerably after finetuning.}
    \label{fig:vis_sota}
    \vspace{-12pt}
\end{figure*}

\subsection{Ablation Study}

We ablate hyper-parameters in the supervised Taskonomy Depth setting, and provide additional comparisons.

\begin{table}[tb]
    \small
    \centering
    \caption{Ablation study on the equivariant loss coefficient.}
    \label{tab:abl_coef}
    \begin{tabular}{lcccc}
    \toprule
        Coefficient & 1e-5 & 3e-5 & \bf 1e-4 & 3e-4 \\
    \midrule
        $\delta\!\!>\!\!1.25$ (\%) & 25.5 & 25.2 & \bf 25.0 & 25.8 \\
    \bottomrule
    \end{tabular}
    \vspace{3mm}
    \caption{Ablation study on which layer to apply equivariant loss.}
    \label{tab:abl_layer}
    \begin{tabular}{lcccccc}
    \toprule
        Layer & L & {\bf L-1} & up0 & up1 & up2 & up3 \\
    \midrule
        Dimension                  & 1    & 16       & 16       & 32   & 64   & 128 \\
        $\delta\!\!>\!\!1.25$ (\%) & 25.7 & \bf 25.0 & \bf 25.0 & 25.2 & 25.1 & 26.1 \\
    \bottomrule
    \end{tabular}
    \vspace{-5pt}
\end{table}

\para{Number of crops (Figure~\ref{fig:abl_num}).}
The optimal number of crops per image $K$ (appears in Eq.~\ref{eq:equi_loss}) for our approach is around 3. We choose 3 in our experiments. Note that in the figure, a single stddev is estimated for all $K$ as the error bars. We also control each run to take roughly the same wall-clock time, which means K=3 yields the best trade-off between extra computing and performance under the fixed computation time budget. The depth error of our approach is almost always below that of the data augmentation alone baseline, indicating the higher efficiency of our approach.

\begin{figure}[tb]
    \centering
    \includegraphics[width=0.7\linewidth]{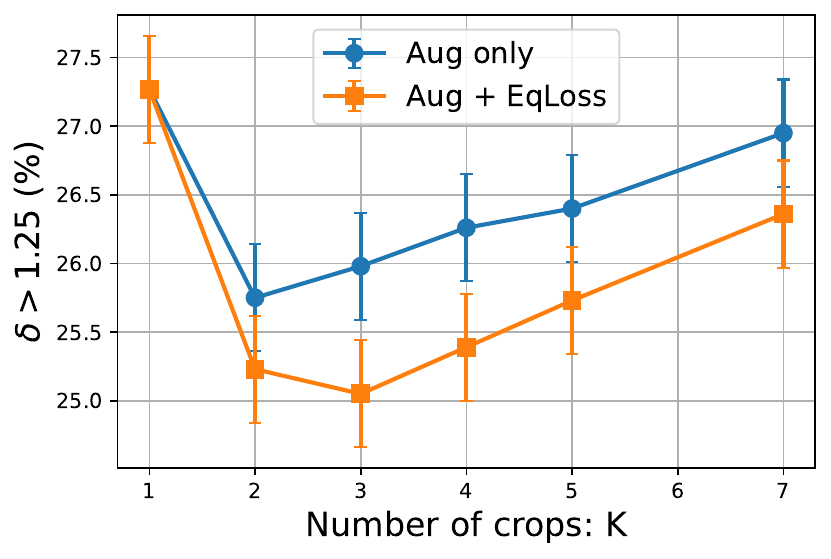}
    \caption{Study on the number of random crops per image.}
    \label{fig:abl_num}
    \vspace{-15pt}
\end{figure}

\para{Equivariant loss coefficient and location.}
In Table~\ref{tab:abl_coef},
3e-5 or 1e-4 performs well; the latter is slightly better. Table~\ref{tab:abl_layer} studies where to put the equivariant loss. L means applying the loss on the final output, L-1 means the penultimate Conv layer (which we use in experiments), up0-3 means the progressively earlier upsampling block of the U-Net. Applying the loss around L-1 seems to be working well, while going deeper into earlier layers yields worse results. This could be due to the lower resolution of early-stage feature maps.

\begin{table}[t]
\small
\centering
\caption{Inference-time equivariant averaging yields only small depth error reduction on NYU-v2, compared to the reduction from our \emph{training-time} equivariant finetuning.}
\label{tab:inference}
\setlength{\tabcolsep}{2.5pt}
\begin{tabular}{@{}lcc}
    \toprule
    $\delta\!\!>\!\!1.25$, AbsRel(\%)  & No averaging  & 3 crops averaging \\
    \midrule
    MiDaS-v2.1 pre-trained        & 8.71, 9.68   & 8.63 (-0.08), 9.61 (-0.07)  \\
    MiDaS-v2.1 + EqLoss  & 7.82, 8.92   & 7.78 (-0.04), 8.86 (-0.06) \\
    \bottomrule
\end{tabular}
\vspace{-5mm}
\end{table}

\para{Inference-time equivariant averaging (Table~\ref{tab:inference}).} We tested both the MiDaS pre-trained network and our finetuned network on NYU-v2. In both cases, inference-time averaging offers a small reduction of depth prediction error (smaller than our equivariant finetuning), which suggests the non-equivariant issue cannot be simply addressed by it.
Note that inference-time averaging increases latency--small improvement at the cost of the multiplied running time. The benefit of our approach is that the workload of averaging is offloaded to training, so that the inference procedure is unchanged and efficient (1 forward pass).

\para{Comparison to contrastive learning.}
In Table~\ref{tab:res_pretrain} of the supplementary material, we {initialize} from DenseCL \cite{wang2021dense} or PixelPro \cite{xie2021propagate} pre-trained ResNet; and in Table~\ref{tab:densecl}, we use DenseCL loss as {regularization during training}. In summary, \ref{tab:res_pretrain} suggests that although DenseCL indeed provides superior performance than random/supervised initialization, it does not completely resolve the equivariance issue--and our method can further improve upon DenseCL initialization; \ref{tab:densecl} suggests that our method (K=3) is stronger than pairwise DenseCL regularization (K=2) during training both depth and normal tasks.

\vspace{-5pt}
\section{Conclusion}

This paper reveals a salient problem in state-of-the-art depth and normal predictors -- that they are not equivariant to cropping, and proposes an equivariant regularization approach to address it. We demonstrate the usefulness of our approach in supervised, semi-supervised and unsupervised settings. We substantially improve equivariance and accuracy of state-of-the-art pre-trained models on NYU-v2 test set without using ground-truth labels. We hope future work can explore the powerful idea of equivariance in other dense prediction tasks and with transformations beyond cropping.

\noindent{\small
{\textbf{Acknowledgement.} This work was supported in part by NSF Grant 2106825, NIFA Award 2020-67021-32799, the Jump ARCHES endowment, the NCSA Fellows program, the Illinois-Insper Partnership, and the Amazon Research Award. This work used NVIDIA GPUs at NCSA Delta through allocations CIS220014 and CIS230012 from the ACCESS program.
Special thanks to Aditya Prakash for helping with the poster presentation.
}

%%%%%%%%% REFERENCES
{\small
\bibliographystyle{ieee_fullname}
\bibliography{main}
}

\def\parasmall#1{\smallskip\noindent\textbf{#1\,}}
\def\para#1{\medskip\noindent\textbf{#1\,}}

\def\gT{{\mathcal{T}}}

\appendix
\counterwithin{table}{section}
\counterwithin{figure}{section}

\section{Quantitative study of equivariant error}
In Figure~\ref{fig:sota_bad}, we qualitatively show that the state-of-the-art depth predictor, MiDaS-v3.0 DPT-Large \cite{ranftl2021vision}, possess insufficient equivariance to cropping transform. Here, we illustrate the same problem in a quantitative manner. From the input image in Figure~\ref{fig:sota_bad}, we generate 5,000 random pairs of crops with scale variation 0.85-1 and aspect ratio variation 3/4-4/3. Note the scale variation is deliberately chosen not to be drastic. We resize them and pass all of them to the pre-trained network to get 5,000 depth map predictions. After that, we compute the AbsRel (absolute relative error of depth, averaged over pixels) between the overlapped region of the pairs of predictions (using one as the target), and call this number $\mathrm{eqerr_{depth}}(f,t_1,t_2)$.  See the following equation. Here, $f$ is the depth predictor, $t_1,t_2$ represent a pair of randomly sampled crop transforms.
\begin{equation*}
    \mathrm{eqerr_{depth}}(f, t_1, t_2) = 
    \mathrm{AbsRel}( t_1^{-1} \circ f \circ t_1(x), t_2^{-1} \circ f \circ t_2(x) ) 
    .
\end{equation*}

This number essentially measures the degree of variation caused by random cropping. A perfectly equivariant predictor will have $\mathrm{eqerr_{depth}}=0$ for any crop transforms $t_1,t_2$. In Figure~\ref{fig:boxp_depth}, we draw the distribution of the 5,000 $\mathrm{eqerr_{depth}}$'s in a box plot, and compare to the AbsRel between the prediction and the ground truth (red line), for both the model before and after our equivariant fine-tuning.

From the left part of Figure~\ref{fig:boxp_depth}, we observe that the variation caused by random cropping (the box plot) is very large compared to the AbsRel to ground truth. The mean of variation is almost 6\%. The largest error can go beyond 10\%, whereas the error against ground truth is just 13.6\%. The right part shows the same quantities after our equivariant fine-tuning. We observe that the the variation caused by cropping is much smaller now, while the accuracy with respect to ground truth also improves.

\begin{figure}[t]
    \centering
    \includegraphics[width=\linewidth]{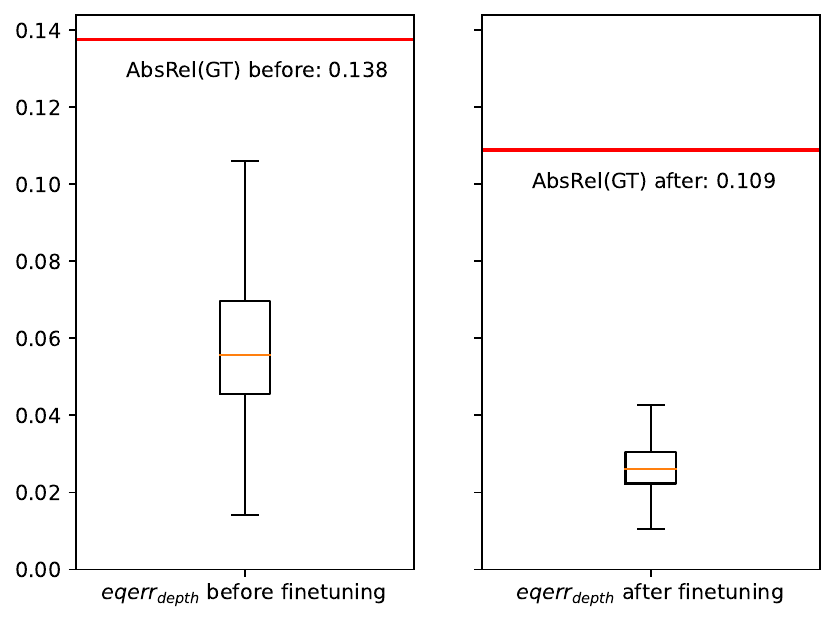}
    \caption{Box plots of variations caused by random cropping, before and after our unsupervised equivariant finetuning, for the example picture in Figure~\ref{fig:sota_bad} of the main paper. Red lines show the error against ground truth. The model is pre-trained MiDaS-v3.0 DPT-Large \cite{ranftl2021vision}. Equivariant fine-tuning shrinks the variation by random cropping considerably.}
    \label{fig:boxp_depth}
\end{figure}

\section{Additional results}

\paragraph{Using ImageNet supervised and dense contrastive learning pre-trained ResNets as initialization.}

The main paper Table~\ref{tab:res_task} shows supervised Taskonomy \cite{zamir2018taskonomy} results with the UNet architecture defined in \cite{zamir2020robust}. Here, we consider ResNet-50 \cite{he2016deep} based encoder-decoder architecture defined in \cite{zamir2018taskonomy} for depth prediction. This architecture uses the ResNet-50 backbone as encoder to obtain a 2048x8x8 feature map for 3x256x256 RGB input, then uses a 10-layer convolutional decoder (with transposed convolutions in last 5 layers) to decode a 256x256 full-sized output. We put our equivariance loss (EqLoss) with K=3 on the second to last layer. 

With ResNet as encoder, we have the possibility to initialize our training with pre-trained models. We examine ImageNet supervised classification pretrained from torchvision \cite{paszke2019pytorch} and PixelPro \cite{xie2021propagate} ImageNet-pretrained model. The reason to study pixel-wise dense contrastive learning pre-trained models such as DenseCL \cite{wang2021dense} and PixelPro (PixPro) \cite{xie2021propagate}, is that, (1) they are pre-trained with dense correspondence between two random crops -- the idea is similar to ours, therefore may achieve higher equivariance to cropping; (2) better suited for downstream dense prediction tasks than image-level pretrained models. We are curious if initializing from these feature stacks would alleviate the equivariance problem in depth predictors.

Table~\ref{tab:res_pretrain} lists the results. We can make several observations: 
(1) Initializing from PixelPro is better than from ImageNet sup., which is in turn better than random init. 
(2) PixelPro initialization does not completely resolve the non-equivariance issue, as the EqLoss of the final depth predictor is still high, which means there is still inconsistency between different crops of the same images. 
(3) Regardless of initialization, adding our EqLoss technique improves the final accuracy and equivariance. The improvement is larger for random init than others. The equivariance for PixelPro is improved as well, measured by lower validation EqLoss.
(4) Our method generalizes to ResNet architecture, in addition to the UNet (Tables~\ref{tab:res_task},\ref{tab:res_semi}) and Dense Prediction Transformer (Table~\ref{tab:res_nyu_depth}, \cite{ranftl2021vision}) architectures in the main paper.

\begin{table}[t]
    \small
    \centering
    \caption{ResNet-50 results on Taskonomy Depth-ZBuffer \cite{zamir2018taskonomy} with random, ImageNet sup., and PixelPro \cite{xie2021propagate} initializations. `+EqLoss' rows are adding our equivariant loss. $\delta\!\!>\!\!1.25$ and AbsRel are validation error metrics (lower the more accurate). `EqLoss' column is the validation equivariant loss (lower the more equivariant). Other experiment settings are the same as Table~\ref{tab:res_task} in the main paper.}
    \label{tab:res_pretrain}
    \setlength{\tabcolsep}{3pt}
    \begin{tabular}{@{}lccccc@{}}
    \toprule
        Pretrain & +EqLoss & $\delta\!\!>\!\!1.25$ (\%)$\downarrow$ & AbsRel (\%)$\downarrow$ &
        EqLoss$\downarrow$ &
        \\
    \midrule
        Random init. &            & 31.6     & 21.0     & 0.485  \\
        Random init. & \checkmark & 30.3 & 20.3 & 0.349  \\
    \midrule
        ImageNet sup. &            & 24.7     & 17.6     & 0.403  \\
        ImageNet sup. & \checkmark & 24.6 & 17.5 & 0.358  \\
    \midrule
        PixelPro \cite{xie2021propagate} & & 22.8     & 17.0     & 0.514  \\
        PixelPro & \checkmark              & 22.7     & 16.7     & 0.463 \\
    \bottomrule
    \end{tabular}
\end{table}

\paragraph{Using dense CL loss during depth/normal network training.}

\begin{table}[t]
\small
\centering
\caption{Compare the K=2 pixel-wise dense contrastive loss variant and the K=3 variant in our main results.}
\setlength{\tabcolsep}{7pt}
\begin{tabular}{@{}lcc@{}}
    \toprule
    & Dense CL (K=2)  & Our EqLoss (K=3) \\
    \midrule
    Depth: $\delta\!\!>\!\!1.25$ (\%)$\downarrow$ & 25.3  & \textbf{25.0} \\
    Normal: Ang error $^\circ$$\downarrow$        & 6.53  & \textbf{6.47} \\
    \bottomrule
\end{tabular}
\label{tab:densecl}
\end{table}

When K=2, our equivariant loss reduces to a type of pixel-wise dense contrastive loss, while our final results use K=3. Pixel-wise dense contrastive methods such as DenseCL \cite{wang2021dense} and PixelPro \cite{xie2021propagate} appear in self-supervised learning literature. They are relevant to our paper because the idea is also to learn equivariant rather than invariant representations in contrastive learning. Pixel-level contrastive learning have shown to improve upon image-level contrastive learning \cite{chen2020simple,he2020momentum} when transferring to detection and segmentation downstream tasks. However, they have not been applied to state-of-the-art depth and normal predictors to the best of our knowledge. Figure~\ref{fig:abl_num} includes this comparison of K=2 and K=3 for depth prediction. In Table~\ref{tab:densecl}, we supplement that result with the surface normal result. In both tasks, the K=3 variant is better than the pixel-wise dense contrastive loss variant, under the same wall-clock time budget. Regardless of the variants, our main point is that equivariance is missing from current depth and normal predictors and the equivariant regularization technique improves the performance of them by increasing equivariance.

\paragraph{Comparison of label and feature space equivariant regularization.}

We show additional surface normal prediction result for the comparison of loss layer in Table~\ref{tab:labelspace} to supplement Table~\ref{tab:abl_layer}. We find applying our EqLoss on feature space is also better than label space for surface normal.

\begin{table}[t]
\small
\centering
\caption{Compare imposing equivariant loss in label space and feature space in supervised Taskonomy \cite{zamir2018taskonomy} settings.}
\setlength{\tabcolsep}{8pt}
\begin{tabular}{@{}lcc@{}}
    \toprule
    & L: Label Space  & L-1: Feature Space \\
    \midrule
    Depth: $\delta\!\!>\!\!1.25$ (\%)$\downarrow$ & 25.7  & \textbf{25.0} \\
    Normal: Ang error $^\circ$$\downarrow$        & 6.54  & \textbf{6.47} \\
    \bottomrule
\end{tabular}
\label{tab:labelspace}
\end{table}

\section{Additional details}

\paragraph{Cosine weighting window.}
In Section 5.1, we mention the use of a weighting window with smooth edges when computing the equivariant average to suppress the boundary artifacts. The motivation is that the boundary predictions (of depths, for example) may not be accurate because the input may not contain enough context for those pixels. It is beneficial to down-weight them in averaging. Figure~\ref{fig:coswin} shows the illustration of the actual weighting window used in our experiments. The smooth edges are generated from a smooth-changing cosine function. The average operation in Eq.~\ref{eq:equi_loss} will become a weighted average.

\begin{figure}[t]
    \centering
    \includegraphics[width=0.5\linewidth]{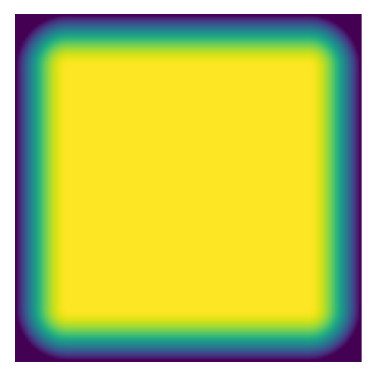}
    \caption{Cosine window weighting map when computing the output average by Eq.~\ref{eq:equi_loss} The brighter pixels mean weights close to 1, and the darker pixels mean weights close to 0. The edge transition follows a cosine function. The output map of each crop is weighted by this map. The boundary outputs will contribute less to the average to avoid artifacts.}
    \label{fig:coswin}
\end{figure}

\paragraph{Linear predictor and stop gradient.} Inspired by contrastive learning with predictor \cite{chen2021exploring,grill2020bootstrap}, we compared equivariant loss with or without a predictor function between the individual crop outputs and the average output. The intuition is that the predictor and stop gradient technique prevents the network from learning a collapsed constant representation. We found that training with a predictor is usually more stable and better-performing. In one experiment, the validation L1 loss improved from 5.61e-2 to 5.49e-2. Our predictor is a linear layer without bias terms, initialized to be the identity function, predicting from each crop output to the average output. The stop gradient is on the average output.

\paragraph{Applying equivariant loss to more than one layer.}
Table~\ref{tab:abl_layer} of the main paper studies the location of our equivariant loss. It is natural to consider applying the loss on more than one layer. We attempted applying it on both L-1 and up1. However, there seems to be no additional gain from this: the $\delta>1.25$ becomes 25.3\%. While we believe there might be more potential in general, we feel applying to multiple layers requires more effort on hyper-parameter tuning and causes unnecessary complexity. We thus stick to applying on only one layer for simplicity.

\paragraph{Special considerations for depth predictors.}
The pre-trained model from \cite{ranftl2020towards,ranftl2021vision} are trained with a combination of loss functions in the disparity space (inverse depth), and the prediction only satisfies $p \approx \alpha + \beta \frac1d$ where $p$ is the predicted disparity and $d$ is the actual depth.
Following their practice, we train with the L1 loss on the disparities (inverse depths) instead of depths.
Due to the same reason, before computing the evaluation metrics, we also follow their practice to use least square regression, i.e., compute the best $\alpha$ and $\beta$, to align the values of predicted disparity map to the ground truth disparity map.

\section{Is it just depth or normal?}

%  \cite{zhu2017unpaired,liu2021learning}
Our paper focuses on state-of-the-art depth and normal prediction models, and finds them not very equivariant to crop transform. Does the problem only exist for depth and normal predictors? We believe the problem is actually quite prevalent and was previously under-explored. In many image-to-image translation tasks where equivariance is desired, the network is not explicitly trained to be equivariant --It does not have strong preference that the output of cropping should not change.

Here, we use the CycleGAN horse-to-zebra translation \cite{zhu2017unpaired} as yet another failure example of equivariance in dense prediction models. Figure~\ref{fig:cyclegan} shows that the resulting stripes on the zebra are sensitive to the crop locations, while ideally the translated image should be a deterministic mapping of the input image contents. Similar issue was observed in the Fig.~7 of \cite{zhang2019making} with the building windows as the example. This problem is even more salient if the method is used to translate a video, as in the official gif example of CycleGAN\footnote{https://github.com/junyanz/pytorch-CycleGAN-and-pix2pix}, where we can visually see the unstable predictions of stripes. Therefore, we believe our approach is broadly relevant to the field. Our equivariant regularization loss can be employed as an additional loss during training to promote equivariance. 

\begin{figure*}[h]
    \centering
    \includegraphics[width=\linewidth]{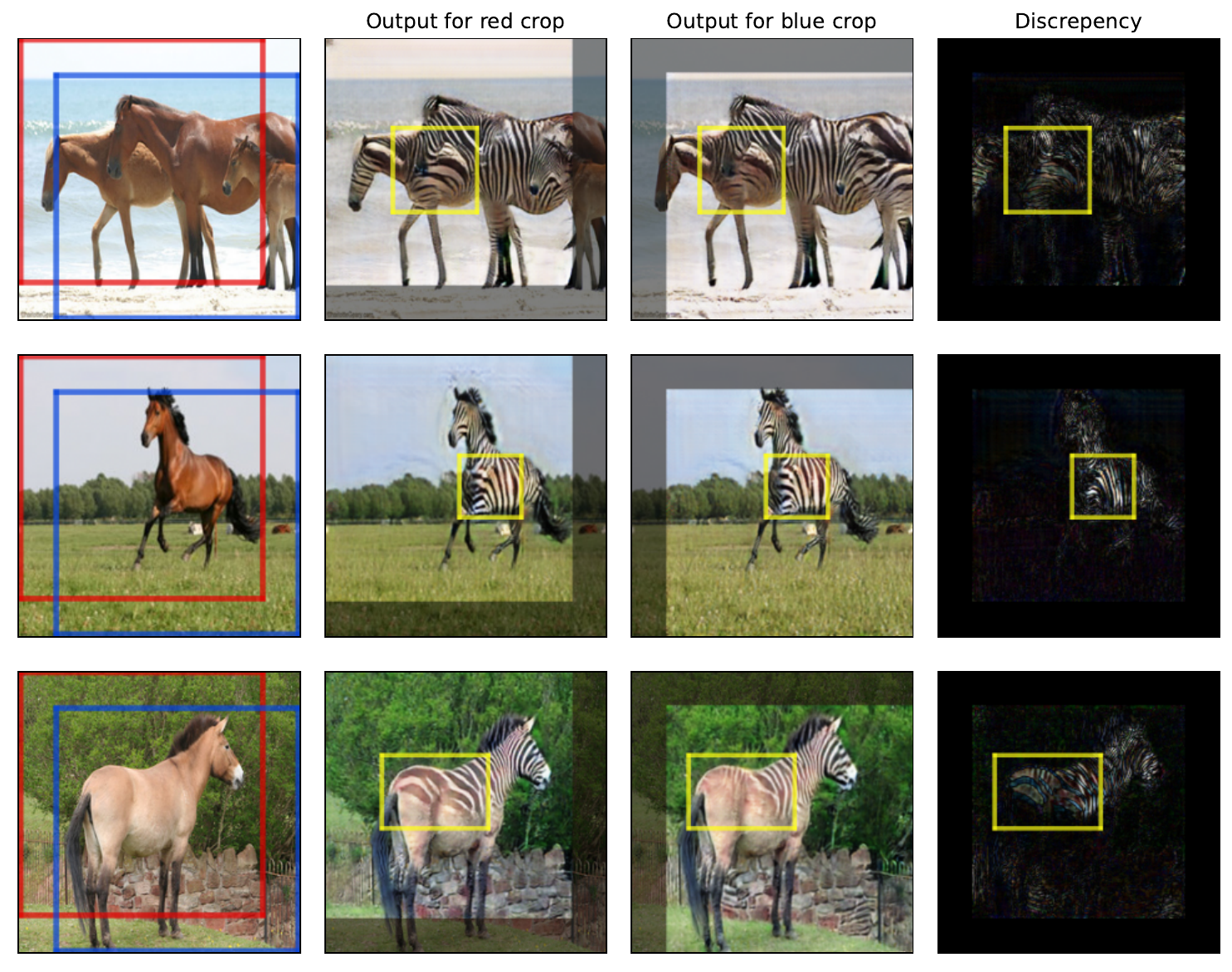}
    \caption{CycleGAN fails to be equivariant in the horse-to-zebra examples, in addition to the depth and surface normal cases of the main paper. We believe the issue of non-equivariance is quite common in image dense prediction models and has been overlooked by prior work (besides in semantic segmentation). Our equivariant regularization approach has potential uses in these other domains.}
    \label{fig:cyclegan}
\end{figure*}

\begin{figure*}[tp]
    \centering
    \includegraphics[width=\linewidth]{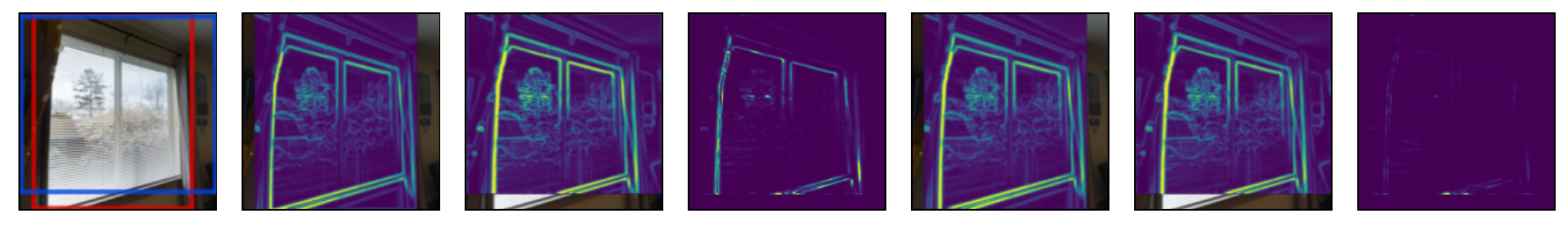}
    \includegraphics[width=\linewidth]{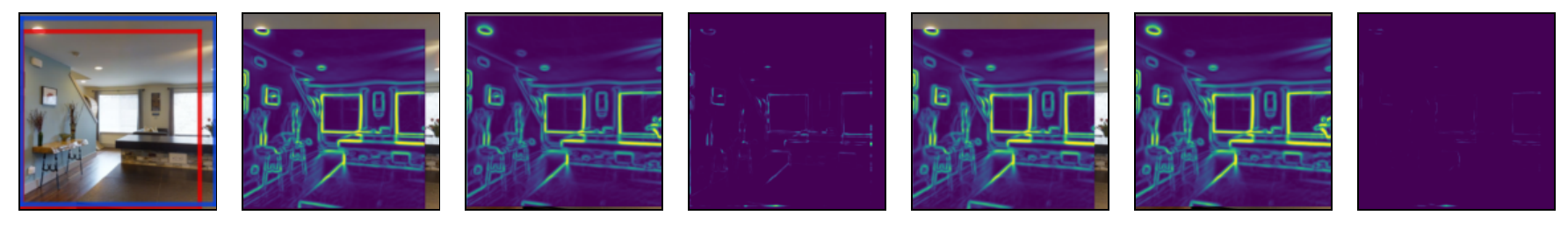}
    \caption{Visualization of Table~\ref{tab:res_task} edge detection results of Taskonomy validation images. From left to right, the columns are images (red, blue crops), predictions of the baseline model (without equivariant loss) and their discrepancy, predictions of our model (with equivariant loss) and their discrepancy.}
    \label{fig:more_vis}
\end{figure*}

\section{More visualization}

As a supplementary visualization to Figure~\ref{fig:vis_sota} (depth and normal) of the main paper, we provide edge detection results on Taskonomy \cite{zamir2018taskonomy} in Figure~\ref{fig:more_vis}. Our model with equivariant loss is more robust to cropping.

\end{document}